\documentclass{article}

\PassOptionsToPackage{numbers,sort&compress}{natbib}
\usepackage[preprint]{neurips_2026}

\usepackage[utf8]{inputenc}
\usepackage[T1]{fontenc}
\usepackage{times}
\usepackage{microtype}
\usepackage{graphicx}
\usepackage{amsmath}
\usepackage{amssymb}
\usepackage{booktabs}
\usepackage{multirow}
\usepackage{tabularx}
\usepackage{array}
\usepackage[table]{xcolor}
\usepackage{hyperref}
\usepackage{enumitem}
\usepackage{caption}
\usepackage{subcaption}
\usepackage{tcolorbox}
\tcbuselibrary{breakable,skins}
\usepackage{listings}
\usepackage{amsthm}
\usepackage{algorithm}
\usepackage{algorithmic}

\usepackage{hypcap}
\usepackage{arydshln}
\usepackage{wrapfig}
\tcbuselibrary{breakable, listings}
\usepackage[dvipsnames]{xcolor}
\usepackage{array}
\usepackage{newunicodechar}
\usepackage{listings}
\usepackage{makecell}
\definecolor{promptbg}{HTML}{F7F9FC}
\definecolor{promptframe}{HTML}{3A6EA5}
\definecolor{taggreen}{HTML}{1A7A3C}
\definecolor{tagpurple}{HTML}{6B2FA0}
\definecolor{tagblue}{HTML}{1A4FA0}
\definecolor{slotbrown}{HTML}{8B4513}

\lstdefinestyle{promptstyle}{
  basicstyle=\ttfamily\small,
  breaklines=true,
  breakatwhitespace=false,
  columns=fullflexible,
  keepspaces=true,
  extendedchars=true,
  frame=none,
  backgroundcolor=\color{promptbg},
  aboveskip=0pt, belowskip=0pt,
  xleftmargin=4pt, xrightmargin=4pt,
}

\newtheorem{definition}{Definition}
\newtheorem{proposition}{Proposition}

\usepackage{enumitem}

\tcbset{
  promptbox/.style={
    enhanced, breakable,
    colback=promptbg,
    colframe=promptframe,
    arc=4pt, outer arc=4pt,
    boxrule=0.8pt,
    left=8pt, right=8pt, top=6pt, bottom=6pt,
    fonttitle=\bfseries\small,
    attach boxed title to top left={yshift=-2mm, xshift=6pt},
    boxed title style={
      colback=promptframe, colframe=promptframe,
      arc=3pt, boxrule=0pt,
      left=4pt, right=4pt, top=2pt, bottom=2pt,
    },
  }
}

\tcbuselibrary{skins, breakable}
\usepackage{fontawesome5}
\usepackage{adjustbox}
\usepackage{xcolor}
\usepackage{booktabs}
\usepackage{hyperref}

\definecolor{agentbg}{RGB}{235, 245, 255}
\definecolor{agentborder}{RGB}{59, 130, 246}
\definecolor{toolbg}{RGB}{255, 251, 235}
\definecolor{toolborder}{RGB}{245, 158, 11}
\definecolor{envbg}{RGB}{243, 244, 246}
\definecolor{envborder}{RGB}{107, 114, 128}

\tcbset{
  envstyle/.style={
    colback=envbg, colframe=envborder,
    fonttitle=\bfseries\small, breakable,
    left=4pt, right=4pt, top=4pt, bottom=4pt
  },
  agentstyle/.style={
    colback=agentbg, colframe=agentborder,
    fonttitle=\bfseries\small, breakable,
    left=4pt, right=4pt, top=4pt, bottom=4pt
  },
  toolstyle/.style={
    colback=toolbg, colframe=toolborder,
    fonttitle=\bfseries\small, breakable,
    left=4pt, right=4pt, top=4pt, bottom=4pt
  },
}

\graphicspath{{figures/}{figures/generated/}}

\title{Rewarding Beliefs, Not Actions: Consistency-Guided Credit Assignment for Long-Horizon Agents}

\author{%
Wenjie Tang$^{1}$ \quad
Minne Li$^{2}$ \quad
Sijie Huang$^{3}$ \quad
Liquan Xiao$^{1}$ \quad
Yuan Zhou$^{2}$\thanks{Corresponding author.}
\\[0.5em]
$^{1}$College of Computer, National University of Defense Technology \\
$^{2}$Intelligent Game and Decision Lab (IGDL) \\
$^{3}$Institute of Artificial Intelligence, Xiamen University
}

\begin{document}
\maketitle

\begin{abstract}

    Reinforcement learning from verifiable rewards (RLVR) is a promising paradigm for improving large language model (LLM) agents on long-horizon interactive tasks. However, in partially observable environments, incomplete observations cause agent beliefs to drift over time, while delayed rewards obscure the causal impact of intermediate decisions, exacerbating temporal credit assignment challenges. To address this, we propose ReBel (Reward Belief), a process-level reinforcement learning algorithm that explicitly models structured belief states to summarize interaction history and guide subsequent policy learning. ReBel introduces belief-consistency supervision, converting discrepancies between predicted beliefs and observed feedback into dense self-supervised signals without requiring external step-wise annotations or verifiers. It also employs belief-aware grouping to compare trajectories under similar belief states, yielding more robust and lower-variance advantage estimates. We evaluate ReBel on challenging long-horizon benchmarks, including ALFWorld and WebShop. ReBel improves task success by up to $20.4$ percentage points over the episode-level baseline GRPO and increases sample efficiency by $2.1\times$. These results suggest that belief-aware self-supervision is a promising direction for reliable long-horizon decision-making under partial observability. Code is available at: \url{https://github.com/Fateyetian/Rebel.git}.

\end{abstract}

\section{Introduction}
\label{sec:introduction}


Large language models (LLMs) are increasingly being deployed as
autonomous agents for long-horizon interactive tasks such as embodied
instruction following and web navigation~
\cite{xi2025agentgymrl, cao2025skyrl, zhou2024webarena}.
These settings require more than fluent generation: agents must maintain
an evolving understanding of the environment, plan over multiple steps,
and adapt robustly under uncertainty
~\cite{yao2024uprop, shorinwa2026uqsurvey, qi2025webrl}.
Reinforcement learning from verifiable rewards (RLVR) has emerged as a
promising paradigm for improving such agents, because it optimizes
policies directly against objective and environment-grounded outcomes
rather than relying on expensive human annotations or potentially
unreliable external judges
~\cite{lambert2024tulu3, guo2025deepseekr1, su2025crossing}.
Compared with LLM-as-a-judge or manually specified intermediate labels,
RLVR provides scalable supervision with clear task-level semantics
~\cite{zheng2023judging, gu2025surveyjudge}, making it particularly
attractive for training agents that must act reliably in open-ended
environments
~\cite{qian2025toolrl,feng2025groupingrouppolicyoptimizationllm}.

Despite its promise, RLVR remains difficult in partially observable
environments, where agents must continuously infer latent context from
incomplete observation histories rather than direct access to the full
state~\cite{chen2025reinforcement}. Small inference errors can
accumulate over time, inducing belief drift: a growing mismatch between
the agent's internal state estimate and the environment~\cite{liu2025advanceschallengesfoundationagents, shorinwa2026uqsurvey}.
As illustrated in Figure~\ref{fig:intro}, an agent may confidently
believe that the key is in a box, only to observe after opening it that
the box is actually empty. If the policy fails to revise this belief, it
may repeatedly inspect the wrong location and drive the trajectory
progressively away from the goal~\cite{yao2024uprop}. Delayed terminal
rewards further obscure which intermediate belief errors caused the
failure, making temporal credit assignment especially challenging. Thus,
effective policies must learn not only which actions to take, but also
when their latent state estimates should be trusted~\cite{DU2024CausalGuidedAL}.

Existing RLVR methods are not well suited to this setting.
Episode-level rewards are often too sparse to correct intermediate
belief errors, especially when failures are only revealed at the end of
a long trajectory~\cite{chen2025sparse2densekeypointdrivengenerativeframework, fu2026rewardshapingmitigatereward}.
Step-wise supervision or external verifiers can provide denser feedback,
but they are costly, difficult to scale, and often unavailable in
open-ended interactive tasks~\cite{Choi2025InPlaceFA}.
More importantly, most existing approaches optimize trajectories
primarily through return-based signals, without explicitly modeling the
agent's evolving belief state~\cite{yao2024uprop, shorinwa2026uqsurvey}.
This makes it difficult to distinguish failures caused by poor action
selection from failures caused by incorrect state estimation~\cite{Lidayan2025ABBELLA}.
In partially observable long-horizon tasks, these two factors are
tightly coupled, and ignoring belief dynamics can lead to unstable
optimization and poor sample efficiency~\cite{kim2025sampleefficiency, guo2025deepseekr1}.

To address this challenge, we propose \textbf{ReBel} (\textbf{Re}ward-\textbf{Bel}ief), a process-level reinforcement learning framework for partially observable long-horizon tasks. The core insight is that effective credit assignment in partially observable settings requires more than reward signals and demands explicit supervision of belief dynamics that bridge observations and actions. Failures in long-horizon tasks often stem from persistent internal misunderstandings rather than poor action selection. Concretely, ReBel realizes this idea through two complementary design layers. First, it maintains a structured belief representation that summarizes the agent's inferred environment state from past interaction history. Rather than attempting to reconstruct a complete latent world model, this representation captures task-relevant state predicates that are sufficient for decision making. Second, ReBel introduces belief-consistency supervision, which converts mismatches between predicted beliefs and subsequent environment observations into dense self-supervised feedback. This provides process-level correction signals without requiring step-wise annotations or external verifiers. In addition, we design a belief-aware grouping strategy that compares trajectories under similar belief states, enabling more robust advantage estimation and reducing optimization variance. Together, these components form a closed iterative learning loop in which more accurate beliefs induce more informative comparisons, which in turn stabilize policy updates and further improve belief quality.

We evaluate ReBel on two partially observable long-horizon benchmarks: ALFWorld~\cite{ALFWorld20} for embodied task planning under incomplete observations and WebShop~\cite{yao2022webshop} for goal-driven web interactions under sparse delayed feedback. Using Qwen2.5-1.5B-Instruct~\cite{yang2024qwen25}, ReBel achieves success rates of \(93.2 \pm 4.1\%\) on ALFWorld~\cite{ALFWorld20} and \(75.1 \pm 2.7\%\) on WebShop~\cite{yao2022webshop}, improving over the strongest episode-level baseline GRPO~\cite{shao2024deepseekmath} by \(20.4\) percentage points and over the strongest step-level method \(\mathrm{GiGPO}_{\mathrm{w/o\ std}}\)~\cite{feng2025groupingrouppolicyoptimizationllm} by \(7.1\) and \(7.7\) percentage points on ALFWorld~\cite{ALFWorld20} and WebShop~\cite{yao2022webshop}, respectively. ReBel reaches GRPO's~\cite{shao2024deepseekmath} final performance in approximately \(35\) training iterations, yielding a \(2.1\times\) sample-efficiency gain, all without requiring dense human annotations or external step-wise verifiers. These results validate belief-aware RLVR as an effective framework for optimizing LLM agents in partially observable long-horizon settings.

In summary, our main contributions are threefold:
\begin{enumerate}
    \item We identify belief drift under partial observability as a key failure mode of RLVR in long-horizon interactive tasks, where delayed rewards exacerbate temporal credit assignment.
    \item We propose ReBel, a process-level reinforcement learning framework that combines structured belief representation, belief-consistency supervision, and belief-aware grouping to provide dense self-supervised feedback without external step-wise verifiers.
    \item We demonstrate substantial gains on ALFWorld~\cite{ALFWorld20} and WebShop~\cite{yao2022webshop}, showing that ReBel improves both task success and sample efficiency over strong episode-level and step-level baselines.
\end{enumerate}

\section{Related Work}

\paragraph{Reinforcement learning for LLM agents.}
As large language models (LLMs) are increasingly deployed as autonomous agents in program synthesis, embodied control, device interaction, and web navigation, reinforcement learning (RL) has become a central paradigm for improving their decision-making behavior. Early RLHF methods, such as InstructGPT~\cite{NEURIPS2022_b1efde53}, demonstrated that preference-based optimization can effectively align model outputs with human intent. More recent work has shifted from single-turn response optimization toward trajectory-level learning in interactive environments such as ALFWorld~\cite{ALFWorld20} and WebShop~\cite{yao2022webshop}. To improve exploration and long-horizon reasoning, methods such as RAP~\cite{hao-etal-2023-reasoning} and Agent Q~\cite{putta2024agent} incorporate Monte Carlo tree search (MCTS~\cite{swiechowski2023monte}) into training or inference, while ArCHer~\cite{zhou2024archer} introduces history-aware value estimation for agent trajectories. In parallel, critic-free optimization methods such as GRPO~\cite{shao2024deepseekmath} and RLOO~\cite{ahmadian2024rloo} reduce optimization overhead and improve training efficiency. Recent systems including DigiRL~\cite{bai2024digirl}, WebRL~\cite{qi2024webrl}, RAGEN~\cite{wang2025ragen}, and LOOP~\cite{chen2025reinforcement} further demonstrate the promise of RL for end-to-end device control and web interaction. Despite these advances, most existing methods still rely primarily on terminal rewards or assume near-complete observability, which makes them vulnerable to severe credit assignment challenges in partially observable long-horizon tasks.

\paragraph{Process supervision and belief modeling.}
A complementary line of work seeks to densify learning signals by providing supervision at the process level rather than only at the final outcome. Process supervision was advocated in early work by Uesato et al.~\cite{uesato2022solving} and Let's Verify Step by Step~\cite{lightman2023let}, which showed that intermediate verification can substantially improve reasoning and planning. Building on this idea, recent methods such as Math-Shepherd~\cite{wang-etal-2024-math}, PRIME~\cite{cui2025prime}, and Watch Every Step~\cite{xiong2024watch} explore automated or weakly supervised step-level reward construction, reducing reliance on expensive human annotations. More agent-oriented methods, including GiGPO~\cite{feng2025groupingrouppolicyoptimizationllm} and related follow-up work~\cite{liu2025agentic,zhuang2025enhancingagenticrlprogressive}, further refine optimization at the step level by grouping trajectories and estimating advantages over finer-grained units. However, these methods typically condition grouping and reward estimation on the current observation alone, which can be brittle when the underlying environment is only partially observed. In such settings, observation-conditioned grouping may fail to reflect the agent's latent state, leading to degraded grouping quality and unstable advantage normalization. In parallel, work on partially observable decision making has explored belief-centric approaches through retrieval augmentation~\cite{guu2020realm}, symbolic augmentation~\cite{mai2022symbolic}, and broader augmented LM paradigms~\cite{mialon2023augmented}, as well as intrinsic motivation mechanisms such as curiosity~\cite{schmidhuber1991possibility} and random network distillation~\cite{burda2018exploration}. While these methods can improve representation learning or exploration, they do not directly couple belief evolution with process-level optimization. ReBel addresses this gap by turning belief consistency into a dense self-supervised training signal, thereby aligning policy optimization with latent state tracking in partially observable long-horizon tasks.

\section{Method}
\label{sec:method}

\begin{figure*}[t]
  \centering
  \includegraphics[width=\textwidth]{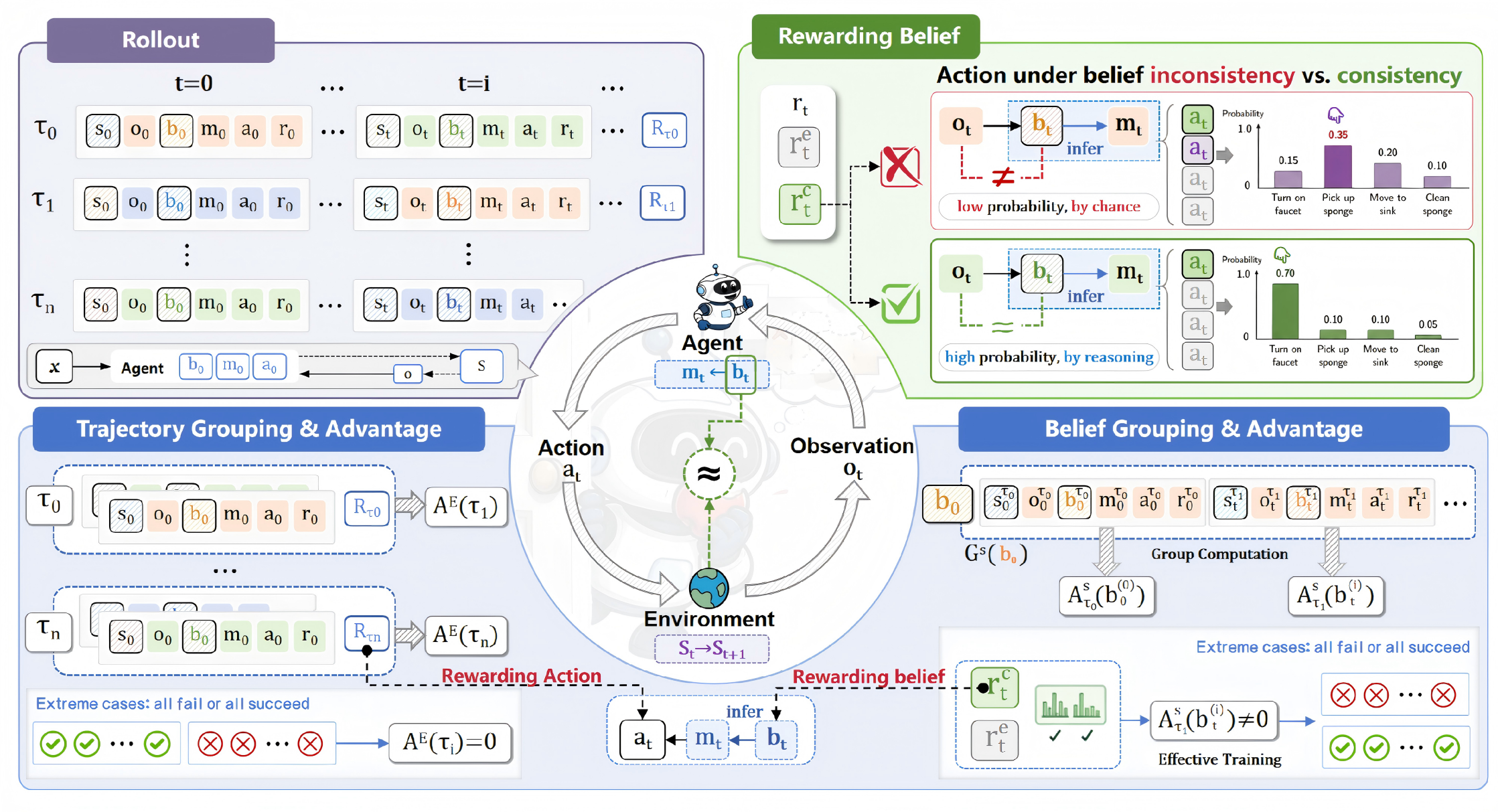}
  \caption{\textbf{Overview of ReBel.}
  ReBel learns belief-aware policies for partially observable long-horizon tasks by making latent belief explicit and decomposing policy generation into belief, think, and action. It turns sparse terminal rewards into step-wise belief consistency feedback and performs belief-anchor grouping to support stable step-level advantage estimation.}
  \label{fig:main}
\end{figure*}

Reinforcement learning with verifiable rewards (RLVR) has significantly improved the performance of large language model (LLM) agents on long-horizon tasks. However, extending this paradigm to partially observable environments remains challenging, primarily because the agent's internal reasoning process is not directly transparent. Outcome-based optimization such as vanilla GRPO~\cite{shao2024deepseekmath} provides only sparse terminal feedback, making it difficult to distinguish a logically sound trajectory from one that merely succeeds by chance. Action-centered process reward models alleviate this sparsity to some extent, but they often remain insensitive to the underlying belief state and may therefore reinforce actions driven by incorrect internal assumptions. We refer to this phenomenon as \textbf{credit drift}, where spurious successes contaminate credit assignment and lead to high-variance updates as well as poor generalization.

To address this issue, we propose \textbf{ReBel} (\textbf{Re}warding \textbf{Bel}ief), a framework that shifts process supervision from what the agent does to why it believes an action should work. ReBel introduces a two-level supervision structure: (i) \textbf{consistency-based supervision}, which penalizes mismatches between the agent's predicted belief state and subsequent observations; and (ii) \textbf{belief-anchored advantage estimation}, which compares actions only within semantically aligned belief clusters. Together, these components enable process-level optimization that explicitly aligns policy learning with latent-state tracking in partially observable long-horizon tasks.

\subsection{Belief-Structured Interaction in POMDPs}

We consider the general setting in which large language model (LLM) agents perform multi-step tasks in a partially observable Markov decision process (POMDP). The process is defined by the tuple \(\left(\mathcal{S},\mathcal{A},\Omega,\mathcal{T},\mathcal{O},\mathcal{R},\gamma\right)\), where \(\mathcal{S}\) denotes the latent state space, \(\mathcal{A}\) the action space, \(\Omega\) the observation space, \(\mathcal{T}\) the transition kernel, \(\mathcal{O}\) the observation model, \(\mathcal{R}\) the reward function, and \(\gamma\) the discount factor. Given a task instruction \(x \in \mathcal{X}\), the agent receives a local observation \(o_t \in \Omega\) at each discrete time step \(t = 1, \ldots, T\). Since the underlying environment state \(s_t \in \mathcal{S}\) is not directly accessible, the agent must reason over a history trace \(h_t = (o_1, a_1, o_2, a_2, \ldots, a_{t-1}, o_t)\), which records both observations and prior decisions. At each interaction step, the policy \(\pi_\theta\) generates a structured output field \(y_t = (b_t, z_t, a_t)\), where \(b_t\) denotes a structured belief, \(z_t\) an intermediate thought state, and \(a_t \in \mathcal{A}\) an executable action. The environment then returns a scalar reward \(r_t \in \mathbb{R}\). In realistic long-horizon scenarios, reward signals are typically extremely sparse and delayed, often appearing only at the end of an episode as a terminal indicator of success or failure. When a single trajectory spans thousands of tokens, credit assignment is no longer merely a sparse-reward problem; rather, it becomes an attribution problem across a long generation chain and multiple internal decisions. In particular, it is often difficult to determine whether the final outcome should be attributed to a specific observation interpretation, an intermediate reasoning step, or a particular action choice.

To address this challenge, ReBel does not directly map history to actions in an end-to-end manner. Instead, it explicitly reconstructs each generation step as a Belief-Think-Action process. Specifically, before executing an action, the agent first constructs a structured belief \(b_t\) from the history \(h_t\) and task instruction \(x\), where \(b_t\) explicitly encodes the agent's assessment of the current environment state. The agent then performs internal reasoning conditioned on \(b_t\) to produce an intermediate thought state \(z_t\), and finally outputs an executable action \(a_t\) based on both \(b_t\) and \(z_t\). This process is factorized as
\begin{equation}
\pi_{\theta}(b_t,z_t,a_t \mid h_t,x)
=
\pi_{\theta}(b_t \mid h_t,x)\,
\pi_{\theta}(z_t \mid h_t,x,b_t)\,
\pi_{\theta}(a_t \mid h_t,x,b_t,z_t).
\label{eq:bta_factorization}
\end{equation}
Here, \(b_t \subseteq P = \{p_1,\ldots,p_K\}\) denotes a task-relevant belief set composed of predefined predicates. It is worth noting that \(b_t\) is not the Bayesian posterior belief in the standard POMDP sense, but rather a symbolic approximation of the current environment state, used to explicitly represent the agent's state assessment. Unlike methods that compress internal reasoning into a single latent variable, we explicitly model belief \(b_t\), thought \(z_t\), and action \(a_t\), thereby making the intermediate cognitive process directly supervisable. Accordingly, training signals can come not only from trajectory-level outcomes, but also from intermediate beliefs, reasoning processes, and action decisions. By organizing the generation process as \(b_t \rightarrow z_t \rightarrow a_t\), ReBel converts the credit assignment problem in long-horizon tasks into local constraints on the consistency between intermediate beliefs and subsequent behaviors, thereby providing a more stable learning signal.

\subsection{Belief Consistency Supervision}

The effectiveness of explicit belief lies not in its static representation, but in whether it can remain consistent with evidence generated through environmental evolution over time. Based on this insight, ReBel introduces an observation-based dense supervision mechanism that decomposes sparse terminal feedback into step-wise belief consistency evaluations, enabling belief quality to be optimized at each time step. Let \(P=\{p_1,\ldots,p_K\}\) denote the set of task-relevant predicates. We represent the belief at time step \(t\) as a binary vector over the predicate space, \(b_t \in \{0,1\}^K\), where \(b_{t,k}=1\) indicates that the agent believes predicate \(p_k\) to be true at time \(t\), and \(b_{t,k}=0\) indicates it to be false. To characterize the consistency between belief and observations, we define a consistency indicator function \(C_k(b_t,a_t,o_{t+1}) \in \{0,1\}\), which determines whether the judgment on predicate \(p_k\) encoded in \(b_t\) can be supported by the next observation \(o_{t+1}\) after executing action \(a_t\). Here, action \(a_t\) influences observations indirectly by inducing state transitions in the environment; therefore, \(C_k\) measures whether the current belief remains supported by evidence after environmental evolution.

Since not all predicates are verifiable at every time step, we further introduce an observability mask \(m_t \in \{0,1\}^K\), where \(m_{t,k}=1\) indicates that predicate \(p_k\) can be checked against observation \(o_{t+1}\), and \(m_{t,k}=0\) indicates that the predicate is temporarily unverifiable at the current step. Based on this mask, the belief consistency reward is defined as
\[
r_t^{\mathrm{cons}}
=
\frac{\sum_{k=1}^{K} m_{t,k} C_k(b_t,a_t,o_{t+1})}{\sum_{k=1}^{K} m_{t,k}}
\quad \text{if} \quad
\sum_{k=1}^{K} m_{t,k} > 0.
\]

When \(\sum_{k=1}^{K} m_{t,k}=0\), the consistency reward is undefined for that step. This treatment avoids interpreting an unverifiable state as inconsistent, ensuring that the supervision signal is derived only from observable evidence. In other words, a time step contributes to consistency supervision only when at least one predicate is checkable; for belief items that are temporarily unobservable, no additional training bias is imposed.

To handle predicates that become verifiable only after a delay, we maintain a pending belief buffer that stores historical belief items for which no validation evidence has yet been obtained. We define
\[
\mathcal{U}_t
=
\left\{
(p_k,t')
\,\middle|\,
b_{t',k}=1,\ 
t' \leq t,\ 
\sum_{\tau=t'}^{t} m_{\tau,k}=0
\right\},
\]
where \(\mathcal{U}_t\) denotes the set of belief items that have not yet been validated by observation up to time step \(t\). If a predicate is not observable when it is generated, the corresponding belief is buffered and excluded from reward computation. Once the predicate becomes observable later, its consistency signal is assigned to the original generation step \(t'\). This mechanism enables temporal credit assignment for delayed evidence beyond single-step observability.

\subsection{Belief-Anchor Step-wise Advantage}

While group-based reinforcement learning reduces variance through relative benchmarking across trajectories, relying solely on terminal feedback in partially observable long-horizon tasks leads to severe \textbf{credit drift}, making it difficult to identify the true contribution of specific steps within a sequence. Although introducing step-wise advantage estimation could refine process-level supervision, it encounters the \textbf{singleton group problem} in POMDP scenarios: since the environment state \(s\) is not directly observable and long-horizon trajectories are highly unique in physical space, different trajectories rarely encounter identical physical states at the same time step. Consequently, state-based grouping often collapses into isolated samples, causing local relative advantage signals to fail due to a lack of comparable benchmarks.

To overcome this bottleneck, ReBel shifts the benchmark for credit assignment from the unobservable physical state \(s\) to an explicit \textbf{belief anchor \(b\)}. The core logic of this shift is that, in partially observable environments, the prerequisite for evaluating action quality is no longer ``external consistency'' but rather ``internal cognitive equivalence.'' As long as the agent's belief predicate sets \(b_t\) are logically equivalent, action choices across different trajectories share a common semantic background for benchmarking. Specifically, let the \(i\)-th trajectory be \(\tau_i = \{(o_{i,t}, b_{i,t}, z_{i,t}, a_{i,t}, r_{i,t}^{\mathrm{env}}, r_{i,t}^{\mathrm{cons}})\}_{t=1}^{T_i}\). We define the step-wise total reward as \(r_{i,t}^{\mathrm{tot}} = r_{i,t}^{\mathrm{env}} + \eta r_{i,t}^{\mathrm{cons}}\), and its corresponding step-wise discounted return as \(G_{i,t}^{\mathrm{step}} = \sum_{k=t}^{T_i} \gamma^{k-t} r_{i,k}^{\mathrm{tot}}\). This return not only reflects the action's contribution to the final task goal but also captures its immediate maintenance of cognitive accuracy through \(r_{i,t}^{\mathrm{cons}}\).

Next, we leverage the symbolic representation of the predicate space \(P\) to aggregate trajectory nodes---which are otherwise isolated in physical state space---into semantically consistent comparison groups. For a belief anchor \(\tilde{b} \in \{0,1\}^K\), its corresponding step-wise group is defined as \(\mathcal{G}_S(\tilde{b}) = \{(i,t) \mid b_{i,t} = \tilde{b}\}\). Since belief representations are semantic abstractions of complex physical states, this anchoring mechanism significantly increases the coverage of comparable samples, thereby resolving the singleton group problem. On this basis, we define the belief-anchored step-wise advantage as:
\begin{equation}
A_S(i,t) = 
\begin{cases} 
0, & |\mathcal{G}_S(b_{i,t})| = 1 \\ 
\frac{G_{i,t}^{\mathrm{step}} - \mu(\mathcal{G}_S(b_{i,t}))}{\sigma(\mathcal{G}_S(b_{i,t})) + \epsilon}, & \text{otherwise}
\end{cases}
\end{equation}
where \(\mu\) and \(\sigma\) denote the mean and standard deviation of returns within the group \(\mathcal{G}_S\), respectively. This mechanism achieves the following core advantages through cognitive benchmarking:

\begin{itemize}[leftmargin=*, itemsep=2pt, topsep=2pt, parsep=0pt]
    \item \textbf{Explicit Decoupling of Decision and Cognition:} By comparing actions within the same belief anchor, the model can tell whether low return comes from wrong reasoning or from a poor action choice. This helps the agent learn the best action under its current belief.
    \item \textbf{Learning Resilience under Sparse Feedback:} When the environment reward is zero for all samples in a group, the relative advantage still captures small differences in belief maintenance. This lets the model strengthen actions that keep the belief correct and reduce belief drift.
\end{itemize}

In summary, Belief-Anchor Step-wise Advantage transforms credit assignment from blind competition over outcomes into precise selection under logically equivalent conditions. It provides a more stable process-level supervision signal that aligns with the nature of decision-making for long-horizon agents in partially observable environments.

\subsection{Belief-aware Policy Optimization}

ReBel optimizes a belief-aware policy by combining trajectory-level outcome signals with belief-anchor step-wise advantages. For the $i$-th trajectory $\tau_i$ and time step $t=1,\ldots,T_i$, we define the total advantage as $A^{\mathrm{tot}}_{i,t} = A_E(\tau_i) + \omega A_S(i,t)$, where $A_E(\tau_i)$ captures the relative terminal quality of the trajectory within its group, $A_S(i,t)$ measures the relative quality of the current step under the same belief anchor, and $\omega \geq 0$ balances the two terms.

\begin{equation}
\begin{aligned}
J(\theta)
=
\mathbb{E}_{\tau_i \sim \pi_{\theta_{\mathrm{old}}}}
\Bigg[
\frac{1}{T_i}
\sum_{t=1}^{T_i}
\Big(
&\min\big(
\rho_{i,t}(\theta)A^{\mathrm{tot}}_{i,t},
\operatorname{clip}(\rho_{i,t}(\theta),1-\epsilon,1+\epsilon)A^{\mathrm{tot}}_{i,t}
\big) \\
&-\beta D_{\mathrm{KL}}\big(\pi_\theta(\cdot \mid h_{i,t}) \Vert \pi_{\mathrm{ref}}(\cdot \mid h_{i,t})\big)
\Big)
\Bigg].
\end{aligned}
\end{equation}

Here, $\epsilon$ controls the clipping range and $\beta$ regularizes the deviation from the reference policy $\pi_{\mathrm{ref}}$. The importance ratio is defined over the joint generation of belief, thought, and action as

\begin{equation}
\rho_{i,t}(\theta)
=
\frac{\pi_\theta(b_{i,t},z_{i,t},a_{i,t}\mid h_{i,t})}
{\pi_{\theta_{\mathrm{old}}}(b_{i,t},z_{i,t},a_{i,t}\mid h_{i,t})}.
\end{equation}

Using the Belief-Think-Action factorization, $\pi_\theta(b_{i,t},z_{i,t},a_{i,t}\mid h_{i,t}) = \pi_\theta(b_{i,t}\mid h_{i,t})\pi_\theta(z_{i,t}\mid h_{i,t},b_{i,t})\pi_\theta(a_{i,t}\mid h_{i,t},b_{i,t},z_{i,t})$, this objective propagates learning signals to belief formation, internal reasoning, and action selection, rather than to the final action alone. As a result, ReBel assigns credit not only to successful outcomes, but also to belief-consistent intermediate decisions, enabling temporally grounded credit assignment in partially observable long-horizon tasks.

\section{Experiments}
\label{sec:experiments}

We structure our experiments around three questions: (i) whether \textsc{ReBel} outperforms strong baselines in partially observable environments; (ii) how much each of its main components contributes to performance; and (iii) whether \textsc{ReBel} improves sample efficiency and mitigates belief drift in long-horizon decision making.

\subsection{Experimental Setup}
\label{subsec:setup}

\begin{table*}[t]
\centering
\caption{Performance comparison on ALFWorld~\cite{ALFWorld20} and WebShop~\cite{yao2022webshop}.
  All RL methods are initialized from Qwen2.5-1.5B-Instruct~\cite{yang2024qwen25};
  results are reported as mean\textsubscript{$\pm$std} over 3 random seeds.
  \textbf{Bold}: best result among RL-trained methods.
  ALFWorld~\cite{ALFWorld20} reports per-task and overall success rate (\%);
  WebShop~\cite{yao2022webshop} reports average score and task success rate (\%).
  $\dagger$: initialized with SFT warm start on belief-annotated data;
  the isolated contribution of SFT warm start is quantified in Appendix~\ref{subsec:app_decomp}.}
\label{tab:main_results}
\resizebox{\textwidth}{!}{%
\begin{tabular}{llccccccc|cc}
\toprule
\multirow{2}{*}{Paradigm} & \multirow{2}{*}{Method}
  & \multicolumn{7}{c|}{\textbf{ALFWorld}}
  & \multicolumn{2}{c}{\textbf{WebShop}} \\
  & & Pick & Look & Clean & Heat & Cool & Pick2 & Overall
  & Score & SR \\
\midrule
\multicolumn{11}{l}{\textit{Closed-source frontier models (zero-shot reference)}} \\
Prompt & GPT-4o         & 75.3 & 60.8 & 31.2 & 56.7 & 21.6 & 49.8 & 48.0 & 31.8 & 23.7 \\
Prompt & Gemini-2.5-Pro & 92.8 & 63.3 & 62.1 & 69.0 & 26.6 & 58.7 & 60.3 & 42.5 & 35.9 \\
\midrule
\multicolumn{11}{l}{\textit{Prompt-based agents (Qwen2.5-1.5B-Instruct)}} \\
Prompt & Qwen2.5   & 5.9  & 5.5  & 3.3  & 9.7  & 4.2  & 0.0  & 4.1  & 23.1 & 5.2  \\
Prompt & ReAct     & 17.4 & 20.5 & 15.7 & 6.2  & 7.7  & 2.0  & 12.8 & 40.1 & 11.3 \\
Prompt & Reflexion & 35.3 & 22.2 & 21.7 & 13.6 & 19.4 & 3.7  & 21.8 & 55.8 & 21.9 \\
\midrule
\multicolumn{11}{l}{\textit{RL-trained agents (Qwen2.5-1.5B-Instruct, 3 seeds)}} \\
RL & PPO  & 64.8\textsubscript{$\pm$3.5} & 40.5\textsubscript{$\pm$6.9} & 57.1\textsubscript{$\pm$4.9} & 60.6\textsubscript{$\pm$6.6} & 46.4\textsubscript{$\pm$4.0} & 47.4\textsubscript{$\pm$1.9} & 54.4\textsubscript{$\pm$3.1} & 73.8\textsubscript{$\pm$3.0} & 51.5\textsubscript{$\pm$2.9} \\
RL & RLOO & 88.3\textsubscript{$\pm$3.0} & 52.8\textsubscript{$\pm$8.6} & 71.0\textsubscript{$\pm$5.9} & 62.8\textsubscript{$\pm$8.7} & 66.4\textsubscript{$\pm$5.5} & 56.9\textsubscript{$\pm$4.7} & 69.7\textsubscript{$\pm$2.5} & 73.9\textsubscript{$\pm$5.6} & 52.1\textsubscript{$\pm$6.7} \\
RL & GRPO & 85.3\textsubscript{$\pm$1.5} & 53.7\textsubscript{$\pm$8.0} & 84.5\textsubscript{$\pm$6.8} & 78.2\textsubscript{$\pm$7.9} & 59.7\textsubscript{$\pm$5.0} & 53.5\textsubscript{$\pm$5.6} & 72.8\textsubscript{$\pm$3.6} & 75.8\textsubscript{$\pm$3.5} & 56.8\textsubscript{$\pm$3.8} \\
RL & GiGPO$_{\text{w/std}}$   & 94.4\textsubscript{$\pm$5.9} & 67.5\textsubscript{$\pm$4.6} & \textbf{94.8}\textsubscript{$\pm$3.8} & 94.4\textsubscript{$\pm$7.8} & 79.8\textsubscript{$\pm$4.7} & 76.4\textsubscript{$\pm$5.4} & 86.7\textsubscript{$\pm$1.7} & 83.1\textsubscript{$\pm$1.6} & 65.0\textsubscript{$\pm$3.2} \\
RL & GiGPO$_{\text{w/o std}}$ & \textbf{96.0}\textsubscript{$\pm$1.4} & 76.5\textsubscript{$\pm$3.9} & 91.8\textsubscript{$\pm$5.5} & 91.3\textsubscript{$\pm$6.3} & 71.7\textsubscript{$\pm$8.4} & 79.5\textsubscript{$\pm$7.7} & 86.1\textsubscript{$\pm$4.7} & \textbf{83.5}\textsubscript{$\pm$1.8} & 67.4\textsubscript{$\pm$4.5} \\
\rowcolor{gray!12}
RL & \textbf{ReBel (Ours)}$\dagger$
   & 91.8\textsubscript{$\pm$2.1} & \textbf{84.2}\textsubscript{$\pm$4.3} & 91.3\textsubscript{$\pm$3.1} & \textbf{95.8}\textsubscript{$\pm$4.3} & \textbf{84.8}\textsubscript{$\pm$4.2} & \textbf{96.5}\textsubscript{$\pm$3.2} & \textbf{93.2}\textsubscript{$\pm$4.1} & 79.8\textsubscript{$\pm$2.8} & \textbf{75.1}\textsubscript{$\pm$2.7} \\
\bottomrule
\end{tabular}%
}
\end{table*}

We evaluate \textsc{ReBel} on ALFWorld~\cite{ALFWorld20} and WebShop~\cite{yao2022webshop}, two representative benchmarks for partially observable decision making. We compare against closed-source frontier models, prompt-based open-source agents, and RL-trained baselines, covering zero-shot prompting, in-context reasoning, and policy optimization settings. All methods use Qwen2.5-1.5B-Instruct~\cite{yang2024qwen25} as the backbone model, and \textsc{ReBel} is initialized with an SFT warm start before RL training to stabilize the structured \(\langle\texttt{belief}\rangle\) output format. Unless otherwise specified, all results are reported as mean \(\pm\) standard deviation. Additional dataset statistics, baseline descriptions, hyperparameters, and implementation details are provided in Appendix~\ref{sec:app_exp_details}.

\subsection{Main Results}
\label{subsec:main_results}

Table~\ref{tab:main_results} shows that \textsc{ReBel} establishes a new performance frontier, achieving an overall success rate of \(93.2 \pm 4.1\%\) on ALFWorld~\cite{ALFWorld20} and \(75.1 \pm 2.7\%\) on WebShop~\cite{yao2022webshop}. Compared with the strongest episode-level baseline, GRPO~\cite{shao2024deepseekmath}, this corresponds to improvements of \(+20.4\) and \(+18.3\) percentage points, respectively. Although closed-source frontier models such as Gemini-2.5-Pro and GPT-4o exhibit strong general reasoning capabilities, their performance on these long-horizon partially observable tasks remains substantially lower than that of RL-trained agents, suggesting that generic zero-shot reasoning is insufficient for maintaining state coherence over long trajectories. Prompt-based agents such as ReAct and Reflexion perform even worse, indicating that without parameter updates, errors accumulate rapidly in these environments.

\textsc{ReBel} also consistently outperforms the strongest step-level baseline, \(\text{GiGPO}_{\text{w/o std}}\)~\cite{feng2025groupingrouppolicyoptimizationllm}, by \(+7.1\) points on ALFWorld~\cite{ALFWorld20} and \(+7.7\) points on WebShop~\cite{yao2022webshop}. This gain suggests that belief-based credit assignment provides a more semantically stable learning signal than observation-hash grouping, which can conflate distinct latent states under surface-level similarity. In contrast, \textsc{ReBel} propagates supervision through structured belief representations, enabling the model to distinguish temporally separated but semantically related decisions.

The benefits of \textsc{ReBel} are further reflected in its sample efficiency and robustness on harder tasks. As shown in Figure~\ref{fig:main_wide}(a), \textsc{ReBel} matches the final performance of GRPO~\cite{shao2024deepseekmath} after only about \(35\) training iterations, corresponding to an approximate \(2.1\times\) improvement in sample efficiency. Moreover, Figure~\ref{fig:main_wide}(b) and (c) show that the performance gap between \textsc{ReBel} and GRPO~\cite{shao2024deepseekmath} grows with task difficulty and trajectory length. In the most challenging ALFWorld~\cite{ALFWorld20} task, \textbf{Pick2}, which requires the longest reasoning chain, \textsc{ReBel} achieves a success rate of \(96.5\%\), outperforming the best GiGPO~\cite{feng2025groupingrouppolicyoptimizationllm} variant by \(+17.0\) points. Overall, these results indicate that structured belief tracking is a key factor behind \textsc{ReBel}'s gains in partially observable long-horizon decision making.

\begin{figure*}[t]
  \centering
  \includegraphics[width=\textwidth]{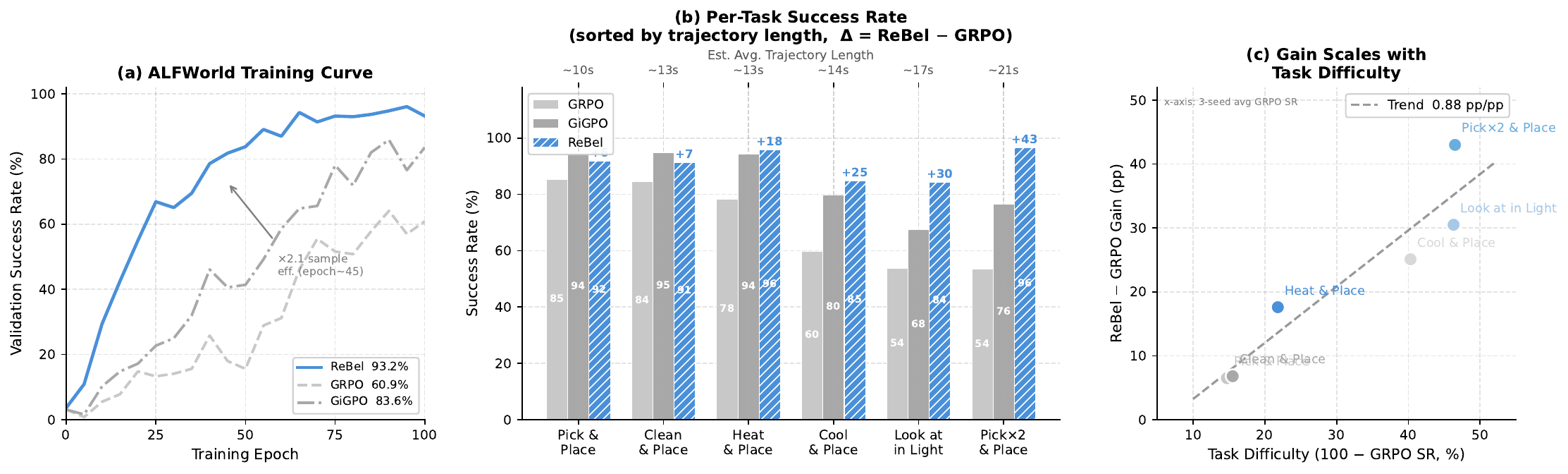}
  \caption{\textbf{Training dynamics and per-task performance.}
  (a) ALFWorld~\cite{ALFWorld20} training curves. \textsc{ReBel} reaches the final GRPO~\cite{shao2024deepseekmath} performance after roughly \(35\) iterations, corresponding to an approximate \(2.1\times\) improvement in sample efficiency.
  (b) Per-task success rates on ALFWorld~\cite{ALFWorld20}, sorted by estimated trajectory length; \(\Delta\) denotes the improvement of \textsc{ReBel} over GRPO~\cite{shao2024deepseekmath}.
  (c) The gain of \textsc{ReBel} over GRPO~\cite{shao2024deepseekmath} increases with task difficulty, indicating that structured belief tracking is more beneficial in harder partially observable tasks.}
  \label{fig:main_wide}
\end{figure*}

\subsection{Ablation Study}
\label{subsec:ablation}

\begin{wraptable}{r}{0.5\linewidth}
\centering
\caption{\textbf{Incremental ablation from GRPO~\cite{shao2024deepseekmath} to full \textsc{ReBel}.}
B0--B3 incrementally add the main modules in a stepwise fashion. B0 is the GRPO~\cite{shao2024deepseekmath} baseline initialized from SFT, and each later variant adds one new component to measure its contribution to the final performance.}
\label{tab:ablation}

\small
\resizebox{\linewidth}{!}{
\begin{tabular}{lccccc}
\toprule
Method & Prompt & Group & \makecell{Step\\Adv.} & \makecell{Belief\\Reward} & ALFWorld \\
\midrule
B0: GRPO & $\times$ & $\times$ & $\times$ & $\times$ & 60.9 \\
B1: + Prompt & $\checkmark$ & $\times$ & $\times$ & $\times$ & 78.1 \\
B2: + Group & $\checkmark$ & $\checkmark$ & $\checkmark$ & $\times$ & 93.0 \\
\textbf{B3: \textsc{ReBel}} & $\checkmark$ & $\checkmark$ & $\checkmark$ & $\checkmark$ & \textbf{96.9} \\
\bottomrule
\end{tabular}
}
\end{wraptable}

Table~\ref{tab:ablation} decomposes the performance gains of \textsc{ReBel} into three key factors: explicit representation, fine-grained credit assignment, and consistency supervision. Moving from vanilla GRPO~\cite{shao2024deepseekmath} (B0) to the belief-prompted variant (B1) yields a 17.2\% improvement, indicating that making latent state tracking explicit through structured reasoning alleviates the burden of encoding history in internal memory. However, although B1 improves the representational form, it still relies primarily on high-variance trajectory-level rewards; as the task horizon increases, such rewards become increasingly insufficient for providing precise optimization guidance.

The most substantial improvement occurs with B2 (93.0\%). This variant mitigates credit diffusion in long-horizon tasks through belief-aligned semantic grouping and step-wise advantage (\(A_{\text{step}}\)). Specifically, by providing dense optimization signals over belief equivalence classes, the model can more accurately identify critical intermediate subgoals. As shown in Figure~2b, this fine-grained feedback is particularly effective in long-horizon tasks such as \textbf{Pickx2 \& Place}, where trajectory-level rewards are often diluted by noise from lengthy action sequences. This mechanism also explains the \(2.1\times\) gain in sample efficiency observed during the early stages of training in Figure~2a.

The full \textsc{ReBel} model (B3, 96.9\%) stabilizes the reasoning--action loop through the auxiliary belief reward \(r^{\text{bel}}\). This reward serves as an alignment anchor, preventing reasoning drift---a phenomenon in which the agent may produce the correct action based on hallucinated state representations. By continuously constraining the policy to remain grounded in the state space defined by environmental predicates, \(r^{\text{bel}}\) improves robustness in partially observable settings, with particularly clear benefits in tasks with greater observation sparsity and deeper reasoning requirements (see Figure~2c). Overall, these results show that the performance gains of \textsc{ReBel} arise from the synergy of structured representation (B1), fine-grained credit assignment (B2), and grounded supervision (B3).

\subsection{Grouping Quality and Efficiency}
\label{subsec:grouping}

\begin{figure*}[t]
  \centering
  \includegraphics[width=\textwidth]{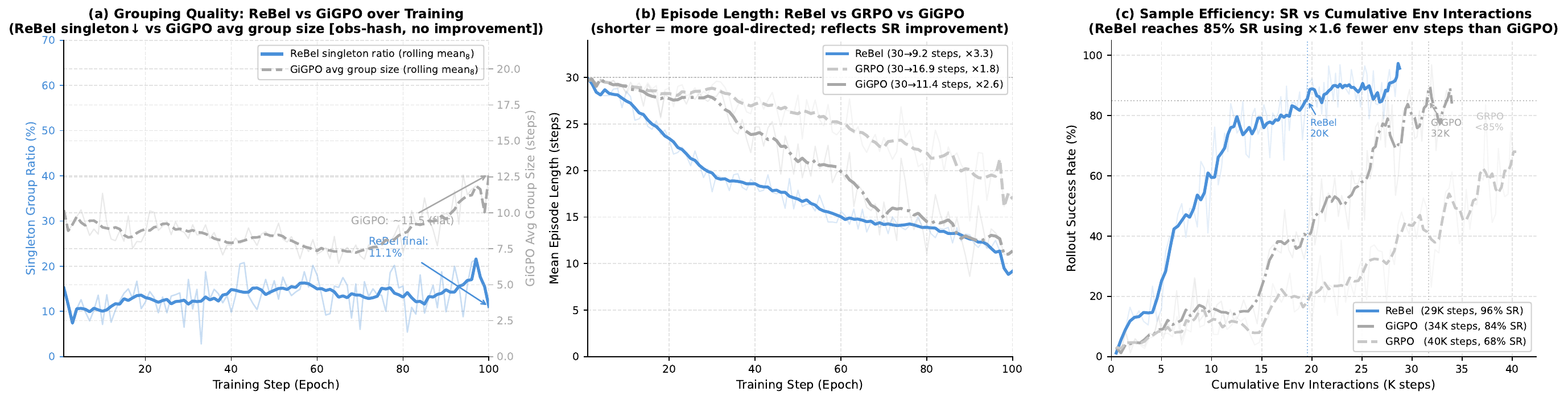}
\caption{\textbf{Grouping quality and training efficiency.}
(a) Singleton ratio for \textsc{ReBel} and GiGPO~\cite{feng2025groupingrouppolicyoptimizationllm} during training. The average group size for GiGPO~\cite{feng2025groupingrouppolicyoptimizationllm} is shown on the right axis.
(b) Average episode length on ALFWorld~\cite{ALFWorld20}. \textsc{ReBel} reduces the average episode length from about 29.9 steps to 9.2 steps, a 3.2$\times$ reduction.
(c) Success rate versus cumulative environment interactions. \textsc{ReBel} reaches 85\% rollout success with roughly 1.6$\times$ fewer environment steps than GiGPO~\cite{feng2025groupingrouppolicyoptimizationllm} and earlier than GRPO~\cite{shao2024deepseekmath}.}
  \label{fig:groupeff}
\end{figure*}

\textbf{Why grouping quality matters: it determines the usefulness of credit assignment.}
As shown in Figure~\ref{fig:groupeff}(a), GiGPO~\cite{feng2025groupingrouppolicyoptimizationllm} produces a relatively high fraction of singleton groups during training, indicating that many states cannot be stably aggregated into semantically coherent sets. An excessive number of singleton groups weakens step-wise normalization and makes advantage estimates more susceptible to noise. In contrast, \textsc{ReBel} consistently maintains a low singleton ratio, while the average group size increases steadily throughout training, suggesting that it can continuously form more stable and semantically coherent belief-based groups. In other words, \textsc{ReBel} does not merely create more groups. It forms groups that are more stable and semantically coherent.

\textbf{Better grouping quality translates into higher task execution efficiency.}
Figure~\ref{fig:groupeff}(b) shows that \textsc{ReBel} reduces the average episode length on ALFWorld~\cite{ALFWorld20} from about 29.9 steps to 9.2 steps, corresponding to a 3.2$\times$ reduction. This means that the model can form effective sub-goal strategies faster and complete tasks with fewer decision steps. By comparison, GiGPO~\cite{feng2025groupingrouppolicyoptimizationllm} yields only modest improvements, while GRPO~\cite{shao2024deepseekmath} shows almost no reduction. This result suggests that higher-quality grouping not only improves credit assignment, but also makes it easier for the policy to learn efficient behavioral paths, such as more accurate object selection and more compact action sequences.

\textbf{At the training level, this advantage further manifests as higher sample efficiency and faster convergence.}
Figure~\ref{fig:groupeff}(c) plots the success rate as a function of cumulative environment interactions. To reach 85\% rollout success, \textsc{ReBel} requires roughly 1.6$\times$ fewer environment steps than GiGPO~\cite{feng2025groupingrouppolicyoptimizationllm} and reaches this level earlier than GRPO~\cite{shao2024deepseekmath}. Meanwhile, \textsc{ReBel}'s convergence curve is smoother and its confidence band is narrower, indicating a more stable training process. Overall, Figures~\ref{fig:groupeff}(a)--(c) form a clear chain. Better belief grouping leads to more reliable credit assignment. This shortens task execution and ultimately improves sample efficiency and convergence.

\section{Conclusion}
\label{sec:conclusion}

We introduce \textsc{ReBel}, a process-supervised reinforcement learning framework that opens a new and promising path toward addressing credit assignment in long-horizon agentic tasks by optimizing explicit belief-state consistency, marking a shift from purely outcome-driven optimization to deeper process-level alignment. Rather than relying solely on terminal rewards, \textsc{ReBel} provides dense step-level supervision through the belief-consistency reward to effectively mitigate cascading failures, while employing a belief-grouping mechanism to identify semantically coherent phases in partially observable environments and enable more precise advantage normalization. Experimental results show that, using only a \(1.5\text{B}\)-parameter backbone, \textsc{ReBel} achieves success rates of \(93.2\%\) on ALFWorld and \(75.1\%\) on WebShop, marking substantial improvements of \(+20.4\) and \(+18.3\) percentage points over the strong GRPO baseline, while enhancing sample efficiency by \(2.1\times\). Our ablation studies further confirm that the integration of belief consistency and semantic grouping produces additive performance gains, underscoring that the self-alignment of an agent's internal world model is essential for achieving efficient and robust autonomous intelligence.

\bibliographystyle{plainnat}
\bibliography{custom}

\newpage

\appendix

\section{Experimental Details}
\label{sec:app_exp_details}

\subsection{Computational Details}

For both ALFWorld and WebShop, we conduct experiments on 4$\times$A800 GPUs using \texttt{Qwen2.5-1.5B-Instruct} as the base model. Each setting is trained for \(100\) iterations. During inference, we use vLLM as the rollout backend with tensor parallelism set to \(1\) and GPU memory utilization fixed at \(0.65\). Unless otherwise specified, all methods are evaluated under the same training budget and environment configuration to ensure a fair comparison.

\subsection{Dataset Details}

\subsubsection{ALFWorld}

We evaluate our method on ALFWorld, a text-based embodied interaction benchmark designed to assess language models in multi-step decision-making scenarios. Unlike static question answering tasks, ALFWorld features a dynamic environment in which the state evolves in response to the agent's actions. At each step, the model must condition on the current observation, historical trajectory, and task goal to decide the next action. As a result, the benchmark not only tests natural language instruction following, but also challenges the model's state tracking, planning, and error correction capabilities in a changing environment.

ALFWorld contains \(3{,}827\) task instances spanning six representative household sub-tasks: Pick \& Place, Examine in Light, Clean \& Place, Heat \& Place, Cool \& Place, and Pick Two \& Place. These sub-tasks differ substantially in action dependency structure, completion conditions, and state-transition patterns. Following the standard protocol, we evaluate on a fixed set of \(128\) held-out tasks, with episode lengths ranging from \(8\) to \(30\) steps. Owing to the diversity in decision horizon and operational complexity across sub-tasks, ALFWorld provides a comprehensive testbed for evaluating generalization in complex embodied environments. From the perspective of ReBel, the key challenge lies in maintaining a persistent understanding of the environment state across multiple interactions and progressively approaching the goal through a sequence of intermediate actions. This makes ALFWorld particularly suitable for assessing whether the model can establish a stable ``belief update--reasoning--action'' decision loop.

\subsubsection{WebShop}

We further evaluate our method on WebShop, a web-based shopping environment designed to test multi-step retrieval, filtering, and decision-making under partial information. In this task, an agent must complete operations such as product search, page navigation, attribute filtering, item comparison, and final purchase within a simulated e-commerce website according to a user instruction.

WebShop contains more than \(1.1\) million products and over \(12{,}000\) user instructions. The maximum episode length is \(15\) steps, and the reward lies in the interval \([0,1]\), with high rewards available only when the agent successfully completes the target. Since product information is often incomplete and the set of available actions changes across pages, WebShop imposes strong requirements on the model's retrieval ability, constraint satisfaction, and multi-step decision-making. From the perspective of ReBel, WebShop exhibits a typical ``reasoning-driven decision-making'' setting: at each step, the model must act based on the current page state and the interaction history, while each intermediate decision directly affects future reachable states and the final return. This makes the benchmark well suited for evaluating whether our method can maintain stable trajectory-level reasoning and policy consistency in complex web environments.

\subsection{Implementation Details}

\textbf{Model and cold-start setup.} We perform all experiments with \texttt{Qwen2.5-1.5B-Instruct}. To ensure that the model acquires basic task-format understanding and structured generation capability before entering the reinforcement learning stage, we adopt an SFT cold-start procedure prior to RL training. Specifically, we first fine-tune the base model on expert trajectories, using \(390\) trajectories for ALFWorld and \(500\) trajectories for WebShop. The SFT stage is trained for \(3\) epochs with a learning rate of \(1 \times 10^{-5}\), and the micro-batch size per GPU is \(8\). The goal of this stage is to teach the model to reliably generate a structured three-part response format consisting of \texttt{<belief>}, \texttt{<think>}, and \texttt{<action>}, thereby providing a format-stable initialization for the subsequent RL stage and reducing invalid exploration caused by formatting errors at the beginning of training.

\textbf{RL optimization setup.} We implement ReBel on top of the veRL-agent framework and introduce the necessary modifications to accommodate partially observable long-horizon tasks. In the RL stage, the learning rate is set to \(1 \times 10^{-6}\), and AdamW is used as the optimizer. For each task input, we sample \(N = 16\) trajectories to form a rollout group, which is used for group-wise normalized advantage estimation. ReBel is trained for \(100\) iterations, whereas the other RL baselines, including PPO, RLOO, GRPO, and GiGPO, are trained for \(150\) iterations and averaged over \(3\) random seeds. For the ablation settings B0--B3, each configuration is trained for \(100\) iterations. The PPO mini-batch size is set to \(256\), and the micro-batch size per GPU is \(4\). The maximum episode length is \(30\) steps for ALFWorld and \(15\) steps for WebShop.

\textbf{Advantage formulation.} ReBel uses a linear combination of two-level advantages, namely the trajectory-level episode advantage \(A^{\mathrm{ep}}(\tau_i)\) and the belief-anchor step-level advantage \(A_S(i,t)\). The former is obtained by normalizing the task success signal within the rollout group, while the latter is computed by normalizing the step-level discounted return within the same belief equivalence class. We set the weight of the step-level component to \(0.5\), yielding \(\hat{A}^{\mathrm{tot}}_{i,t} = A^{\mathrm{ep}}(\tau_i) + 0.5 \cdot A_S(i,t)\).

\textbf{Format regularization.} To encourage the model to produce well-formed three-part responses, we apply a reward penalty of \(-0.1\) to outputs that violate the required format. An output is considered valid only if it contains at least one belief-state tag, namely \texttt{<belief>...\textless/belief>}, and one action tag, namely \texttt{<action>...\textless/action>}.

\textbf{Regularization and discounting.} We set the KL divergence loss coefficient to \(0.01\) using the low-variance estimator, and the KL penalty coefficient to \(0.001\). The RL discount factor is set to \(\gamma_{\mathrm{RL}} = 0.95\).

\subsection{Evaluation Protocol}
We evaluate model performance on the validation set and use task success rate as the primary metric. For ALFWorld, we report both per-subtask success rates and the overall success rate on the fixed set of \(128\) held-out tasks. For WebShop, we report average score, which reflects partial matching quality, as well as success rate \(\mathrm{SR}\). We follow the standard evaluation protocols of each benchmark. During evaluation, sampling is disabled by setting the temperature to \(0\), or equivalently using greedy decoding, in order to more stably reflect the model's generalization performance and to distinguish evaluation from the stochastic rollout policy used during training.

\section{Prompt Templates}
\label{sec:app_prompts}

Figure~\ref{fig:prompt_alfworld} and Figure~\ref{fig:prompt_webshop} present the prompt templates used by \textsc{ReBel} in the ALFWorld and WebShop environments, respectively. The templates are implemented with Python-style string formatting, where placeholders enclosed in curly braces are filled at runtime via \texttt{.format()}. Each prompt includes the task instruction, a short interaction history, the previous belief state, and the current observation. To keep the prompt concise, we retain only the two most recent actions in the history.

All prompts follow the same three-stage decision protocol: (i) belief update, which integrates the previous belief with the current observation; (ii) belief-grounded reasoning, which identifies the current task phase and the most relevant next step; and (iii) action selection, which outputs one executable action from the admissible action set. The model response is structured into three tagged blocks: \texttt{<belief>}, \texttt{<think>}, and \texttt{<action>}. The \texttt{<belief>} block serves as the structured latent state used by \textsc{ReBel}'s grouping-based credit assignment, the \texttt{<think>} block encourages brief intermediate reasoning, and the \texttt{<action>} block contains the final executable action string.

\begin{tcolorbox}[title = {\textbf{ALFWorld Prompt Template}},breakable,
colback=SeaGreen!8!CornflowerBlue!8,
colframe=RoyalPurple!70!Aquamarine,
boxrule=0.6pt,
arc=2mm,
left=1mm,right=1mm,top=1mm,bottom=1mm,
fonttitle=\bfseries]

You are an expert agent operating in the ALFRED embodied environment.

Task: \{task\_description\}

Step count: \{step\_count\} 
Recent history (\{history\_length\} steps): \{action\_history\}

Current step: \{current\_step\} 

Current observation: \{current\_observation\} 

\medskip

Use the previous belief state together with the current observation to
update your belief, reason about the current task phase, and select the
best next action. Historical observations are omitted because their
relevant information is already summarized in the previous belief.

\medskip

Previous belief:
\begin{center}
\textcolor{tagblue}{\textbf{<belief\_prev>}} \quad
\textcolor{slotbrown}{\{previous\_belief\}} \quad
\textcolor{tagblue}{\textbf{</belief\_prev>}}
\end{center}


Respond in the following format:

\tcblower
{\small

\begin{quote}

\textbf{<belief>}

\hspace*{1em} state:
\begin{itemize}
    \item "target": {"[attr]": "[value]"}
    \item "confirmed": {"[attr]": "[verified value]"}
    \item "unconfirmed": ["[attr]"]
    \item "product\_id": "[ASIN or null]"
\end{itemize}

\hspace*{1em} task:
\begin{itemize}
    \item "phase": "[searching | browsing | viewing | selecting | buying | done]"
    \item "target": "[product\_id or null]"
    \item "subgoal": "[updated immediate subgoal]"
\end{itemize}

\hspace*{1em} "prediction":
"If I execute [planned\_action], I expect [next-observation consequence]."

\textbf{</belief>}

\vspace{1mm}

\textbf{<think>}

\hspace*{1em} [Brief reasoning in 2--3 sentences: what is confirmed, what remains
uncertain, which phase the task is in, and why the chosen action is
currently optimal.]

\textbf{</think>}

\vspace{1mm}

\textbf{<action>}

\hspace*{1em} [Exact action string selected from the admissible action set]

\textbf{</action>}

\end{quote}
}

\end{tcolorbox}

\captionof{figure}{ALFWorld prompt template.}
\label{fig:prompt_alfworld}

\begin{tcolorbox}[title = {\textbf{WebShop Prompt Template}},breakable,
colback=SeaGreen!8!CornflowerBlue!8,
colframe=RoyalPurple!70!Aquamarine,
boxrule=0.6pt,
arc=2mm,
left=1mm,right=1mm,top=1mm,bottom=1mm,
fonttitle=\bfseries]

You are an expert autonomous agent operating in the WebShop e-commerce environment.

Task: \{task\_description\}

Step count: \{step\_count\} 

Recent history (\{history\_length\} steps): \{action\_history\}

Current step: \{current\_step\} 

Current observation: \{current\_observation\}

\medskip

Use the previous belief state together with the current observation to
update the belief, infer the current task phase, identify the main
information bottleneck, and select the best next action.

\medskip

Previous belief:
\begin{center}
\textcolor{tagblue}{\textbf{<belief\_prev>}} \quad
\textcolor{slotbrown}{\{previous\_belief\}} \quad
\textcolor{tagblue}{\textbf{</belief\_prev>}}
\end{center}

\medskip

Respond in the following format:

\tcblower

{\small

\begin{quote}

\textbf{<belief>}

\hspace*{1em} "state": 
    \begin{itemize}
        \item "target": \{"[attr]": "[value]"\}
        \item "confirmed": \{"[attr]": "[verified value]"\}
        \item "unconfirmed": ["[attr]"]
        \item "product\_id": "[ASIN or null]"
    \end{itemize}

\hspace*{1em} "task": 
    \begin{itemize}
        \item "phase": "[searching | browsing | viewing | selecting | buying | done]"
        \item "target": "[product\_id or null]"
        \item "subgoal": "[updated immediate subgoal]"
    \end{itemize}

    \item "prediction": \\
    "If I execute [planned\_action], I expect [next-observation consequence]."

\textbf{</belief>}

\vspace{1mm}

\textbf{<think>}

\hspace*{1em} [Brief reasoning in 2--3 sentences: what is confirmed, what remains
uncertain, which phase the task is in, and why the chosen action is
currently optimal.]

\textbf{</think>}

\vspace{1mm}

\textbf{<action>}

\hspace*{1em} [Exact action string selected from the admissible action set]

\textbf{</action>}

\end{quote}
}

\end{tcolorbox}

\captionof{figure}{WebShop prompt template.}
\label{fig:prompt_webshop}

\section{Limitations}
\label{sec:app_limitations}

Our study has two main limitations. First, our experiments are conducted on two representative benchmarks and one backbone scale, 1.5B. This setting allows us to evaluate the core hypothesis of \textsc{ReBel} in a controlled and comparable manner, especially in environments where partial observability and intermediate reasoning play an important role. However, it does not fully cover the diversity of possible task distributions, interaction patterns, and model capacity regimes. For example, tasks with longer horizons, more diverse feedback signals, or substantially larger policy models may exhibit different optimization dynamics. Evaluating \textsc{ReBel} under a wider range of environments and model scales would therefore provide a more complete picture of its empirical generality.

Second, the belief representation used in ALFWorld is designed to align with symbolic object-state tracking, which is a natural fit for the structure of that benchmark. This design choice makes it possible to define and evaluate belief consistency in a transparent and interpretable way. Nevertheless, other domains may expose different forms of observations and hidden states. For instance, vision-based environments may require beliefs over visual entities or spatial layouts, while continuous-control tasks may call for more compact or learned state abstractions. Extending \textsc{ReBel} to such settings would likely require adapting the belief format to the corresponding observation and state spaces. We view this as a promising direction for future work, rather than a fundamental restriction of the framework.

\section{Pseudo Code}

\begin{algorithm}[h]
\caption{ReBel Training}
\label{alg:rebel}
\begin{algorithmic}[1]
\REQUIRE Policy \(\pi_{\theta}\), reference policy \(\pi_{\mathrm{ref}}\), task distribution \(p(x)\), rollout size \(N\), horizon \(T\), weights \(\omega, \beta\), clip range \(\epsilon\)
\STATE Initialize \(\theta \leftarrow \theta_{\mathrm{SFT}}\)
\FOR{each training iteration}
    \STATE Sample task \(x \sim p(x)\) and collect \(N\) rollouts
    \FOR{each rollout \(i \in \{1,\dots,N\}\) and step \(t \in \{1,\dots,T\}\)}
        \STATE Generate \(b_t^{(i)}, z_t^{(i)}, a_t^{(i)}\) from \(\pi_{\theta}(\cdot \mid h_t^{(i)}, x)\)
        \STATE Execute \(a_t^{(i)}\) and observe \(o_{t+1}^{(i)}\)
        \STATE Compute belief-consistency reward \(r_{t,\mathrm{cons}}^{(i)}\) using observability mask and pending buffer
        \STATE Update history \(h_{t+1}^{(i)}\)
    \ENDFOR
    \STATE Compute episode-level advantage \(A_E^{(i)}\)
    \STATE Group step samples by canonicalized belief anchors and compute step-level advantage \(A_S^{(i,t)}\)
    \STATE Combine advantages: \(A_{\mathrm{tot}}^{(i,t)} = A_E^{(i)} + \omega A_S^{(i,t)}\)
    \STATE Update \(\theta\) by maximizing clipped PPO objective with KL regularization:
    \STATE \hspace{1.2em}
    \[
    J(\theta)=\mathbb{E}\!\left[
    \min\!\left(
    r_t(\theta)A_{\mathrm{tot}}^{(i,t)},
    \mathrm{clip}(r_t(\theta),1-\epsilon,1+\epsilon)A_{\mathrm{tot}}^{(i,t)}
    \right)
    \right]
    - \beta\,\mathrm{KL}\!\left(\pi_{\theta}\,\|\,\pi_{\mathrm{ref}}\right)
    \]
\ENDFOR
\end{algorithmic}
\end{algorithm}

\section{Additional Analysis}
\label{sec:app_additional}

\subsection{Belief Drift under Partial Observability}

In partially observable environments, agents must reason over latent beliefs inferred from incomplete observations. As illustrated in Figure~\ref{fig:intro}, a think-only agent may become overconfident in an incorrect belief and continue selecting actions that are locally consistent with this mistaken assumption, even after receiving contradictory observations. Since sparse terminal rewards provide limited supervision over intermediate reasoning errors, standard outcome-level reinforcement learning often fails to correct such latent belief drift. In contrast, belief-aware policies explicitly revise their internal beliefs according to new observations, progressively reducing uncertainty and enabling recovery from incorrect hypotheses before the trajectory diverges from the task objective.


\begin{figure}[h]
  \centering
  \includegraphics[width=\textwidth]{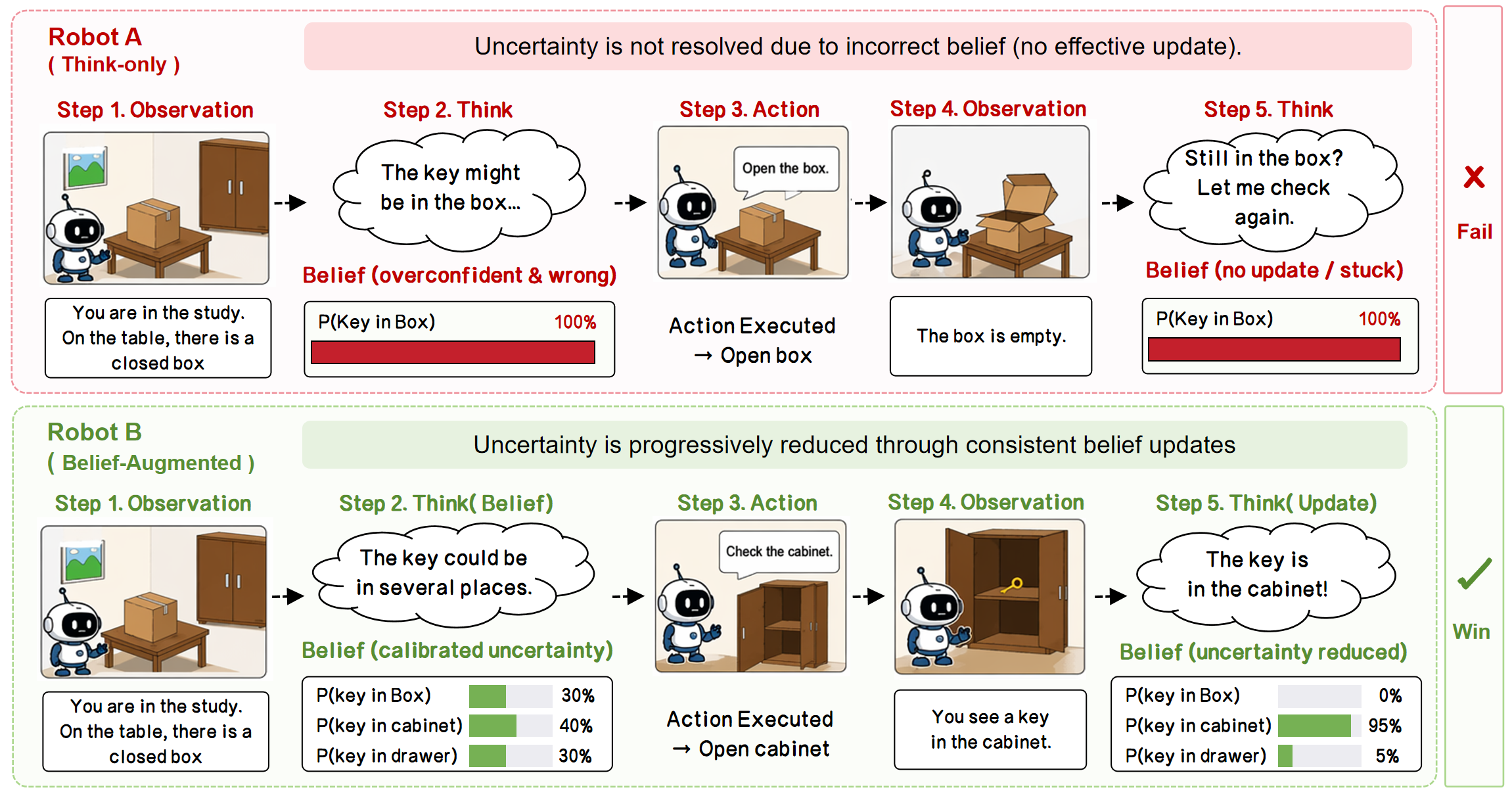}
  \caption{\textbf{Belief drift as a failure mode in partially observable environments.}
    The top row shows a Think-only agent that remains overconfident in an incorrect belief and repeatedly executes invalid actions. The bottom row shows a Belief-Augmented agent that updates its belief from observations, progressively reduces uncertainty, and succeeds in the task. \textsc{ReBel} aims to induce this belief-aware reasoning behavior during RL training.}
  \label{fig:intro}
\end{figure}




\subsection{Independent Contribution of SFT Cold Start}
\label{subsec:app_decomp}

To isolate the contribution of the SFT warm start, we evaluate the agent immediately after SFT initialization, before any RL optimization, and compare it with the full \textsc{ReBel} pipeline after SFT + RL training.

The results indicate that SFT alone provides only limited task performance. On ALFWorld~\cite{ALFWorld20}, the SFT-only model achieves a mean success rate of 3.6\% across three random seeds, with scores ranging from 2.3\% to 5.5\%, which is far below the 93.2\% achieved by full \textsc{ReBel}. This large gap shows that SFT is not the main source of the final performance gain. Instead, its primary role is to stabilize the structured \texttt{<belief>} output format and provide a reliable initialization for subsequent process-level RL optimization.

In other words, SFT mainly serves as a warm start that improves training stability and makes the belief representation usable for downstream reinforcement learning. The substantial performance improvements in \textsc{ReBel} come primarily from RL optimization over belief consistency and semantically coherent grouping, rather than from supervised pretraining itself.

\begin{table}[t]
\centering
\caption{Independent contribution of the SFT warm start on ALFWorld~\cite{ALFWorld20}.
Results are reported as mean $\pm$ std over 3 random seeds.}
\label{tab:sft_ablation}
\resizebox{\columnwidth}{!}{%
\begin{tabular}{llc}
\toprule
Method & Initialization & ALFWorld Overall (\%) \\
\midrule
SFT only (3 epochs, no RL) & SFT checkpoint & 3.6 $\pm$ 1.6 \\
GRPO + SFT init (ablation baseline) & SFT checkpoint & 60.9 \\
\textbf{ReBel} (full, SFT + RL) & SFT checkpoint & \textbf{93.2 $\pm$ 4.1} \\
\bottomrule
\end{tabular}%
}
\end{table}

\subsection{Representative Trajectories}
\label{subsec:app_case}


\section{Theoretical Analysis}
\label{sec:app_theory}

\subsection{Problem Diagnosis and REBEL Design}
\label{subsec:app_problem_design}

We consider a partially observable sequential decision process with latent state \(s_t \in \mathcal{S}\), observation \(o_t \in \mathcal{O}\), action \(a_t \in \mathcal{A}\), and history \(h_t = (o_1, a_1, \ldots, o_t)\). The task input is denoted by \(x\). Since \(s_t\) is not directly observable, a policy must maintain an internal belief over the task-relevant aspects of the environment. In what follows, we use \(\bar b_t\) to denote the implicit belief induced by a conventional end-to-end policy, and \(b_t\) to denote the explicit belief used by REBEL.

A conventional end-to-end policy directly maps the history to the next output,
\[
\pi_\theta(y_t \mid h_t, x),
\]
where \(y_t\) may contain reasoning tokens, tool calls, or the final action. Although this policy implicitly encodes a belief \(\bar b_t\), such belief is neither explicitly exposed nor directly supervised. As a result, if \(\bar b_t\) deviates from the true task-relevant state abstraction, the agent receives no timely correction signal. Under sparse terminal rewards, this deviation may persist across many steps and propagate through subsequent reasoning and action choices.

We formalize this failure mode as belief drift.

\begin{definition}[Belief drift]
\label{def:belief_drift}
Let \(D:\mathcal{B}\times\mathcal{B}\rightarrow \mathbb{R}_{\ge 0}\) be a discrepancy function. The belief drift at step \(t\) is defined as
\[
d_t = D(b_t, \phi(s_t)),
\]
where \(b_t\) is the belief used by the agent for decision making and \(\phi:\mathcal{S}\rightarrow\mathcal{B}\) is a task-relevant projection from latent states to belief space.
\end{definition}

From an optimization perspective, the main difficulty is credit assignment. With only an episode-level return \(R_i\) for trajectory \(i\), all intermediate decisions share the same delayed learning signal:
\[
g_i^{\mathrm{trad}}
=
\sum_{t=1}^{T_i}
\nabla_\theta \log \pi_\theta(y_{i,t}\mid h_{i,t},x_i)\,A_i^E.
\]
Therefore, the policy gradient cannot distinguish whether a failure is caused by an incorrect belief update, an unhelpful reasoning step, or a poor final action. In long-horizon partially observable tasks, this ambiguity makes training unstable. Under the usual weak-correlation assumption across steps, the variance of such a trajectory-level estimator grows at least approximately linearly with horizon length, which explains why sparse terminal supervision is often insufficient.

REBEL addresses this issue by making the latent decision structure explicit. Instead of treating each step as a monolithic generation process, REBEL decomposes it into Belief, Think, and Action layers:
\[
\pi_\theta(b_t, z_t, a_t \mid h_t, x)
=
\pi_\theta(b_t \mid h_t, x)
\,
\pi_\theta(z_t \mid h_t, x, b_t)
\,
\pi_\theta(a_t \mid h_t, x, b_t, z_t).
\]
Here \(b_t\) is the explicit belief, \(z_t\) is the intermediate reasoning trace, and \(a_t\) is the action. This factorization turns the hidden belief into an observable and trainable object, and it separates three sources of error: belief construction, reasoning conditioned on belief, and action selection conditioned on both.

To supervise the belief layer, REBEL introduces belief consistency supervision. Suppose the belief is represented by \(K\) task predicates \(p_1,\ldots,p_K\). Let \(\hat p_{t,k}(b_t)\) be the value of predicate \(p_k\) implied by belief \(b_t\), and let \(v_{t+1,k}\) be the corresponding value verified from the next observation whenever it is observable. Because not every predicate is immediately observable, we define an observability mask \(m_{t,k}\in\{0,1\}\). The consistency indicator is
\[
C_{t,k}
=
\mathbf{1}\{\hat p_{t,k}(b_t)=v_{t+1,k}\}.
\]
The step-wise consistency reward is
\[
r_t^{\mathrm{cons}}
=
\frac{
\sum_{k=1}^{K} m_{t,k} C_{t,k}
}{
\max\{1,\sum_{k=1}^{K} m_{t,k}\}
}.
\]
This reward provides immediate supervision for the observable part of the belief. For predicates that are not yet observable, REBEL maintains a pending buffer \(U_t\) and evaluates them once a later observation makes them verifiable. In this way, consistency supervision supports both immediate correction and delayed verification.

Finally, REBEL introduces a belief-anchor step-wise advantage to avoid the singleton problem in partially observable environments. State-based grouping is ineffective because the same physical state \(s_t\) may rarely reappear across sampled trajectories. REBEL instead groups samples by a symbolic belief anchor \(\tilde b_t\), defined as the observable and pending predicate signature extracted from \(b_t\). The belief-anchor group is
\[
\mathcal{G}_B(\tilde b)
=
\{(i,t): \tilde b_{i,t}=\tilde b\}.
\]
Let \(G_{i,t}^{S}\) be a local step-level return, for example the discounted sum of consistency rewards over a short horizon:
\[
G_{i,t}^{S}
=
\sum_{\ell=t}^{t+H}\gamma^{\ell-t}r_{i,\ell}^{\mathrm{cons}}.
\]
The belief-anchor advantage is then
\[
A_{i,t}^{S}
=
\frac{
G_{i,t}^{S}-\mu_B(\tilde b_{i,t})
}{
\sigma_B(\tilde b_{i,t})+\varepsilon
},
\]
where
\[
\mu_B(\tilde b)
=
\frac{1}{|\mathcal{G}_B(\tilde b)|}
\sum_{(j,u)\in\mathcal{G}_B(\tilde b)} G_{j,u}^{S},
\]
and \(\sigma_B(\tilde b)\) is the corresponding within-anchor standard deviation. The total learning signal combines trajectory-level success with belief-anchor step-level feedback:
\[
A_{i,t}^{\mathrm{tot}}
=
A_i^{E}
+
\omega A_{i,t}^{S}.
\]
This design gives REBEL two complementary supervision channels: \(A_i^E\) preserves global task optimization, while \(A_{i,t}^{S}\) provides dense belief-level correction under a shared semantic context.

\subsection{Variance Reduction Analysis}
\label{subsec:app_variance_analysis}

We now explain why the above design reduces the variance of step-wise advantage estimation. The key idea is that belief consistency supplies a dense intermediate signal, while belief-anchor grouping provides a meaningful comparison set in partially observable settings.

We assume that the belief discrepancy decomposes over task predicates:
\[
D(b_t,\phi(s_t))
=
\sum_{k=1}^{K}w_k \,\ell_k\bigl(\hat p_{t,k}(b_t),p_k(\phi(s_t))\bigr),
\]
where \(w_k\ge 0\) and each \(\ell_k\) is a bounded predicate-level loss. Let
\[
\kappa_t = \frac{1}{K}\sum_{k=1}^{K}m_{t,k}
\]
be the observable predicate fraction at step \(t\).

\begin{proposition}[Consistency supervision contracts observable belief drift]
\label{prop:consistency_contracts_drift}
Assume that observable predicate verification is sound up to error \(\delta_{\mathrm{obs}}\), meaning that whenever \(m_{t,k}=1\) and \(C_{t,k}=1\), the corresponding predicate-level belief error is at most \(\delta_{\mathrm{obs}}\). Suppose further that every pending predicate is verified within at most \(\Delta\) steps with probability at least \(1-\rho\). Then there exist constants \(\alpha>0\), \(c_1>0\), and \(c_2>0\) such that
\[
\mathbb{E}[d_{t+1}]
\le
(1-\alpha\kappa_t)\mathbb{E}[d_t]
+
c_1\,\mathbb{E}[1-r_t^{\mathrm{cons}}]
+
c_2(\rho+\Delta_{\mathrm{err}}),
\]
where \(\Delta_{\mathrm{err}}\) denotes the maximum additional drift accumulated during delayed verification.
\end{proposition}

\begin{proof}
For observable predicates, the consistency reward directly penalizes disagreement between the belief-implied predicate value and the value verified from the subsequent observation. Hence, a higher \(r_t^{\mathrm{cons}}\) implies a smaller observable component of \(D(b_t,\phi(s_t))\), up to verification error \(\delta_{\mathrm{obs}}\).

For unobservable predicates, REBEL does not force an immediate penalty. Instead, the pending buffer \(U_t\) stores such predicates until they become verifiable. By assumption, verification occurs within \(\Delta\) steps with probability at least \(1-\rho\). During this delay, the uncorrected component can accumulate at most \(\Delta_{\mathrm{err}}\) additional error. Combining these two effects yields the stated bound.
\end{proof}

Proposition~\ref{prop:consistency_contracts_drift} shows that consistency supervision acts as an intermediate corrective signal. If the average observability is bounded below by \(\bar\kappa>0\) and the average inconsistency is bounded by \(\epsilon_{\mathrm{cons}}\), then repeated application gives
\[
\limsup_{t\rightarrow\infty}\mathbb{E}[d_t]
\le
\frac{
c_1\epsilon_{\mathrm{cons}}+c_2(\rho+\Delta_{\mathrm{err}})
}{
\alpha\bar\kappa
}.
\]
Thus REBEL prevents belief drift from growing unchecked and confines the expected drift to a bounded region determined by observability, verification quality, and delay.

We next analyze the variance of step-wise advantage estimation. Let \(N\) be the number of sampled step instances in a training batch. For any grouping variable \(H\), define the group containing sample \(j=(i,t)\) as
\[
\mathcal{G}_H(H_j)
=
\{j': H_{j'}=H_j\},
\]
with group size \(n_H(j)=|\mathcal{G}_H(H_j)|\). A state-based method uses \(H_j=s_{i,t}\), while REBEL uses \(H_j=\tilde b_{i,t}\).

The failure of state-based grouping in POMDPs follows from support size. Let \(M_S\) be the effective number of reachable latent states and \(M_B\) be the number of possible belief anchors. If the anchor contains \(K\) ternary predicate values, namely true, false, and unknown, then
\[
M_B\le 3^K.
\]
In long-horizon partially observable tasks, \(M_S\) is typically much larger than \(M_B\). Under a roughly uniform occupancy approximation, the expected group sizes satisfy
\[
\mathbb{E}[n_S]\approx 1+\frac{N-1}{M_S},
\qquad
\mathbb{E}[n_B]\ge 1+\frac{N-1}{M_B}.
\]
Therefore, when \(M_B\ll M_S\), belief-anchor groups are much larger than state-based groups, and state-based groups are often close to singletons.

This difference directly affects advantage variance. Let \(G_j\) denote the step-level target return for sample \(j\). For a grouping variable \(H\), define the centered advantage
\[
\tilde A_j^{H}
=
G_j-\mathbb{E}[G_j\mid H_j].
\]
By the law of total variance,
\[
\mathrm{Var}(\tilde A_j^{H})
=
\mathbb{E}\bigl[\mathrm{Var}(G_j\mid H_j)\bigr].
\]
Hence any grouping variable that captures meaningful semantic similarity reduces the residual variance of the advantage by removing the predictable group-level component.

In practice, the conditional mean is estimated from samples. Using a leave-one-out group mean \(\hat\mu_{H,-j}\), define
\[
\hat A_j^{H}
=
G_j-\hat\mu_{H,-j}.
\]
If the within-group variance is \(\sigma_H^2(H_j)=\mathrm{Var}(G_j\mid H_j)\) and \(n_H(j)>1\), then
\[
\mathrm{Var}(\hat A_j^{H}\mid H_j,n_H(j))
=
\sigma_H^2(H_j)
\left(
1+\frac{1}{n_H(j)-1}
\right).
\]
This expression shows why state-based grouping is unstable in POMDPs: when \(n_S(j)\approx 1\), the group baseline is either unavailable or highly noisy. In contrast, belief-anchor grouping has larger \(n_B(j)\), so the finite-sample estimation term is much smaller.

Belief anchors also reduce the intrinsic within-group variance. Samples sharing the same anchor \(\tilde b\) have the same predicate-level cognitive context, even if their underlying physical states differ. Suppose the local return is Lipschitz with respect to belief drift up to bounded noise:
\[
G_j
=
F(\tilde b_j,a_j)
+
\xi_j,
\qquad
\mathrm{Var}(\xi_j\mid \tilde b_j)\le \sigma_\xi^2,
\]
and
\[
|F(\tilde b_j,a_j)-F(\tilde b_j,a_j')|
\le
L_R d_j.
\]
Then
\[
\mathrm{Var}(G_j\mid \tilde b_j)
\le
\sigma_\xi^2
+
L_R^2\mathbb{E}[d_j^2\mid \tilde b_j].
\]
By Proposition~\ref{prop:consistency_contracts_drift}, consistency supervision bounds the drift term and therefore tightens the within-anchor return distribution.

Combining these observations gives the main conclusion: belief-anchor grouping yields a lower-variance step-wise advantage estimator than state-based grouping in partially observable tasks, because it avoids the singleton-group problem, increases the effective comparison set size, and aligns samples by semantic belief rather than exact latent state. The total learning signal \(A_{i,t}^{\mathrm{tot}}=A_i^E+\omega A_{i,t}^{S}\) further preserves trajectory-level optimization while injecting low-variance step-level correction.

\section{Trajectory Case Study}

\begin{center}
\begin{tabular}{ll}
\toprule
\textbf{ID} & \texttt{b3\_rebel\_00082} \\
\textbf{Task Type} & pick\_clean\_then\_place\_in\_recep \\
\textbf{Status} & \textcolor{green!60!black}{\textbf{Success}} \\
\textbf{Steps} & 8 \\
\textbf{Format Valid Rate:} & 1.00 \\
\textbf{Episode ID:} & 82 \\
\bottomrule
\end{tabular}
\end{center}

\subsection*{Task}
\begin{tcolorbox}[colback=yellow!10, colframe=yellow!60!black, fonttitle=\bfseries\small, title={\faFlag~Task Description}]
put a clean mug in coffeemachine.
\end{tcolorbox}

\subsection*{Interaction}

\begin{tcolorbox}[title={\faEye~Observation (Step 1)}, envstyle]
\small\ttfamily -= Welcome to TextWorld, ALFRED! =-

You are in the middle of a room. Looking quickly around you, you see a cabinet 6, a cabinet 5, a cabinet 4, a cabinet 3, a cabinet 2, a cabinet 1, a coffeemachine 1, a countertop 3, a countertop 2, a countertop 1, a drawer 3, a drawer 2, a drawer 1, a fridge 1, a garbagecan 1, a microwave 1, a shelf 3, a shelf 2, a shelf 1, a sinkbasin 1, a stoveburner 4, a stoveburner 3, a stoveburner 2, a stoveburner 1, and a toaster 1.

Your task is to: put a clean mug in coffeemachine.
\end{tcolorbox}

\begin{tcolorbox}[title={\faRobot~Agent (Step 1) \textcolor{green!60!black}{(admissible)}}, agentstyle]
\textbf{\faDatabase~Belief State:}\\
\begin{adjustbox}{max width=\linewidth}
\parbox{\linewidth}{\small \textbf{Phase:} find \\
    \textbf{Target:} mug \\
    \textbf{Subgoal:} Find the mug so it can later be cleaned and placed in the coffeemachine. \\
    \textbf{Prediction:} If I execute go to countertop 1, I expect On the countertop 1, you see a apple 2, a apple 1, a bread 1, a bread 2, a egg 1, a fork 1, a mug 1, a pan 1, and a tomato 1.}
\end{adjustbox}
\medskip
\textbf{\faBolt~Action:}\\
\texttt{go to countertop 1}
\end{tcolorbox}

\begin{tcolorbox}[title={\faEye~Observation (Step 2)}, envstyle]
\small\ttfamily You arrive at countertop 1. On the countertop 1, you see a apple 3, a cellphone 1, a creditcard 1, a peppershaker 3, a peppershaker 1, and a tomato 1.
\end{tcolorbox}

\begin{tcolorbox}[title={\faRobot~Agent (Step 2) \textcolor{green!60!black}{(admissible)}}, agentstyle]
\textbf{\faDatabase~Belief State:}\\
\begin{adjustbox}{max width=\linewidth}
\parbox{\linewidth}{\small \textbf{Phase:} find \\
    \textbf{Target:} mug \\
    \textbf{Subgoal:} Find the mug so it can later be cleaned and placed in the coffeemachine. \\
    \textbf{Visited:} countertop 1 \\
    \textbf{Prediction:} If I execute go to countertop 2, I expect On the countertop 2, you see a apple 2, a bread 1, a bread 2, a egg 2, a fork 2, a mug 2, a pan 2, and a tomato 2.}
\end{adjustbox}
\medskip
\textbf{\faBolt~Action:}\\
\texttt{go to countertop 2}
\end{tcolorbox}

\begin{tcolorbox}[title={\faEye~Observation (Step 3)}, envstyle]
\small\ttfamily You arrive at countertop 2. On the countertop 2, you see a apple 2, a bowl 1, a butterknife 1, a fork 1, a mug 1, a spoon 1, and a statue 1.
\end{tcolorbox}

\begin{tcolorbox}[title={\faRobot~Agent (Step 3) \textcolor{green!60!black}{(admissible)}}, agentstyle]
\textbf{\faDatabase~Belief State:}\\
\begin{adjustbox}{max width=\linewidth}
\parbox{\linewidth}{\small \textbf{Phase:} pickup \\
    \textbf{Target:} mug 1 \\
    \textbf{Subgoal:} Pick up mug 1 from countertop 2 so it can be cleaned next. \\
    \textbf{Objects:} mug 1: countertop 2 \\
    \textbf{Visited:} countertop 1, countertop 2 \\
    \textbf{Prediction:} If I execute take mug 1 from countertop 2, I expect You pick up the mug 1 from the countertop 2.}
\end{adjustbox}
\medskip
\textbf{\faBolt~Action:}\\
\texttt{take mug 1 from countertop 2}
\end{tcolorbox}

\begin{tcolorbox}[title={\faEye~Observation (Step 4)}, envstyle]
\small\ttfamily You pick up the mug 1 from the countertop 2.
\end{tcolorbox}

\begin{tcolorbox}[title={\faRobot~Agent (Step 4) \textcolor{green!60!black}{(admissible)}}, agentstyle]
\textbf{\faDatabase~Belief State:}\\
\begin{adjustbox}{max width=\linewidth}
\parbox{\linewidth}{\small \textbf{Phase:} transform \\
    \textbf{Target:} mug 1 \\
    \textbf{Subgoal:} Bring the mug 1 to a sinkbasin so it can be cleaned. \\
    \textbf{Objects:} mug 1: in\_hand \\
    \textbf{Visited:} countertop 1, countertop 2 \\
    \textbf{Prediction:} If I execute go to sinkbasin 1, I expect On the sinkbasin 1, you see a apple 2, a apple 1, a bread 1, a fork 2, a fork 1, a knife 1, a mug 2, a potato 1, and a tomato 2.}
\end{adjustbox}
\medskip
\textbf{\faBolt~Action:}\\
\texttt{go to sinkbasin 1}
\end{tcolorbox}

\begin{tcolorbox}[title={\faEye~Observation (Step 5)}, envstyle]
\small\ttfamily You arrive at sinkbasin 1. On the sinkbasin 1, you see a glassbottle 2, and a spatula 1.
\end{tcolorbox}

\begin{tcolorbox}[title={\faRobot~Agent (Step 5) \textcolor{green!60!black}{(admissible)}}, agentstyle]
\textbf{\faDatabase~Belief State:}\\
\begin{adjustbox}{max width=\linewidth}
\parbox{\linewidth}{\small \textbf{Phase:} transform \\
    \textbf{Target:} mug 1 \\
    \textbf{Subgoal:} Clean the mug 1 using sinkbasin 1 so it can later be placed in the coffeemachine. \\
    \textbf{Objects:} mug 1: in\_hand \\
    \textbf{Visited:} countertop 1, countertop 2, sinkbasin 1 \\
    \textbf{Prediction:} If I execute clean mug 1 with sinkbasin 1, I expect On the sinkbasin 1, you see a apple 2, a apple 1, a bread 1, a fork 2, a fork 1, a knife 1, a mug 2, a potato 1, and a tomato 2.}
\end{adjustbox}
\medskip
\textbf{\faBolt~Action:}\\
\texttt{clean mug 1 with sinkbasin 1}
\end{tcolorbox}

\begin{tcolorbox}[title={\faEye~Observation (Step 6)}, envstyle]
\small\ttfamily You clean the mug 1 using the sinkbasin 1.
\end{tcolorbox}

\begin{tcolorbox}[title={\faRobot~Agent (Step 6) \textcolor{green!60!black}{(admissible)}}, agentstyle]
\textbf{\faDatabase~Belief State:}\\
\begin{adjustbox}{max width=\linewidth}
\parbox{\linewidth}{\small \textbf{Phase:} place \\
    \textbf{Target:} mug 1 \\
    \textbf{Subgoal:} Carry the cleaned mug 1 to coffeemachine 1 and place it inside. \\
    \textbf{Objects:} mug 1: in\_hand \\
    \textbf{Visited:} countertop 1, countertop 2, sinkbasin 1 \\
    \textbf{Prediction:} If I execute go to coffeemachine 1, I expect On the coffeemachine 1, you see a apple 1, a bread 2, a bread 1, a butterknife 2, a butterknife 1, a egg 1, a fork 3, a fork 2, a fork 1, a mug 2, a mug 1,...[truncated]}
\end{adjustbox}
\medskip
\textbf{\faBolt~Action:}\\
\texttt{go to coffeemachine 1}
\end{tcolorbox}

\begin{tcolorbox}[title={\faEye~Observation (Step 7)}, envstyle]
\small\ttfamily You arrive at coffeemachine 1. On the coffeemachine 1, you see nothing.
\end{tcolorbox}

\begin{tcolorbox}[title={\faRobot~Agent (Step 7) \textcolor{red}{(inadmissible)}}, agentstyle]
\textbf{\faDatabase~Belief State:}\\
\begin{adjustbox}{max width=\linewidth}
\parbox{\linewidth}{\small \textbf{Phase:} place \\
    \textbf{Target:} mug 1 \\
    \textbf{Subgoal:} Place the cleaned mug 1 into coffeemachine 1. \\
    \textbf{Objects:} mug 1: in\_hand \\
    \textbf{Visited:} countertop 1, countertop 2, sinkbasin 1, coffeemachine 1 \\
    \textbf{Prediction:} If I execute put mug 1 in/on coffeemachine 1, I expect On the coffeemachine 1, you see a apple 1, a bread 2, a bread 1, a butterknife 2, a butterknife 1, a egg 1, a fork 3, a fork 2, a fork 1, a mug 2...[truncated]}
\end{adjustbox}
\medskip
\textbf{\faBolt~Action:}\\
\texttt{put mug 1 in/on coffeemachine 1}
\end{tcolorbox}

\begin{tcolorbox}[title={\faEye~Observation (Step 8)}, envstyle]
\small\ttfamily On the coffeemachine 1, you see nothing.
\end{tcolorbox}

\begin{tcolorbox}[title={\faRobot~Agent (Step 8) \textcolor{green!60!black}{(admissible)}}, agentstyle]
\textbf{\faDatabase~Belief State:}\\
\begin{adjustbox}{max width=\linewidth}
\parbox{\linewidth}{\small \textbf{Phase:} place \\
    \textbf{Target:} mug 1 \\
    \textbf{Subgoal:} Place the cleaned mug 1 into coffeemachine 1. \\
    \textbf{Objects:} mug 1: in\_hand \\
    \textbf{Visited:} countertop 1, countertop 2, sinkbasin 1, coffeemachine 1 \\
    \textbf{Prediction:} If I execute move mug 1 to coffeemachine 1, I expect On the coffeemachine 1, you see a apple 1, a bread 2, a bread 1, a butterknife 2, a butterknife 1, a egg 1, a fork 3, a fork 2, a fork 1, a mug 2, ...[truncated]}
\end{adjustbox}
\medskip
\textbf{\faBolt~Action:}\\
\texttt{move mug 1 to coffeemachine 1}
\end{tcolorbox}

\end{document}